\documentclass[sigconf]{acmart}
\AtBeginDocument{%
  \providecommand\BibTeX{{%
    \normalfont B\kern-0.5em{\scshape i\kern-0.25em b}\kern-0.8em\TeX}}}

\copyrightyear{2023}
\acmYear{2023}
\setcopyright{acmlicensed}\acmConference[KDD '23]{Proceedings of the 29th ACM SIGKDD Conference on Knowledge Discovery and Data Mining}{August 6--10, 2023}{Long Beach, CA, USA}
\acmBooktitle{Proceedings of the 29th ACM SIGKDD Conference on Knowledge Discovery and Data Mining (KDD '23), August 6--10, 2023, Long Beach, CA, USA}
\acmPrice{15.00}
\acmDOI{10.1145/3580305.3599774}
\acmISBN{979-8-4007-0103-0/23/08}
\usepackage{algorithm}
\usepackage{algpseudocode}
\usepackage{natbib}
\setcitestyle{numbers}
\usepackage{graphicx}
\usepackage{caption}
\usepackage{balance}
\usepackage[utf8]{inputenc} 
\usepackage[T1]{fontenc}    
\usepackage{hyperref}       
\usepackage{url}            
\usepackage{booktabs}       
\usepackage{amsfonts}       
\usepackage{nicefrac}       
\usepackage{microtype}      
\usepackage{xcolor}         
\usepackage{amsmath}
\usepackage{soul}
\usepackage{tabu}
 \newcommand{\ind}{\perp\!\!\!\!\perp}

\newtheorem{proposition}{Proposition}
\newtheorem*{proposition*}{Proposition 1}

\newtheorem{assumption}{Assumption}

\usepackage{appendix}

\begin{abstract}
Extracorporeal membrane oxygenation (ECMO) is an essential life-supporting modality for COVID-19 patients who are refractory to conventional therapies. However, the proper treatment decision  has been the subject of significant debate and it remains controversial about who benefits from this scarcely available and technically complex treatment option. To support clinical decisions, it is a critical need to predict the treatment need and the potential treatment and no-treatment responses. Targeting this clinical challenge, we propose Treatment Variational AutoEncoder (TVAE), a novel approach for individualized treatment analysis. TVAE is specifically designed to address the modeling challenges like ECMO with strong treatment selection bias and scarce treatment cases. TVAE conceptualizes the treatment decision as a multi-scale problem. We model a patient's potential treatment assignment and the factual and counterfactual outcomes as part of their intrinsic characteristics that can be represented by a deep latent variable model. The factual and counterfactual prediction errors are alleviated via a reconstruction regularization scheme together with semi-supervision, and the selection bias and the scarcity of treatment cases are mitigated by the disentangled and distribution-matched latent space and the label-balancing generative strategy. We evaluate TVAE on two real-world COVID-19 datasets: an international dataset collected from 1651 hospitals across 63 countries, and a institutional dataset collected from 15 hospitals. The results show that TVAE outperforms state-of-the-art treatment effect models in predicting both the propensity scores and factual outcomes on heterogeneous COVID-19 datasets. Additional experiments also show TVAE outperforms the best existing models in individual treatment effect estimation on the synthesized IHDP benchmark dataset.
\end{abstract}



\begin{document}
\title{Assisting Clinical Decisions for Scarcely Available Treatment via Disentangled Latent Representation}
\author{Bing Xue}
\orcid{0000-0002-9162-098X}
\affiliation{%
  \institution{McKelvey School of Engineering\\Washington University in St. Louis}
  \streetaddress{1 Brookings Drive}
 \city{St. Louis}
 \state{Missouri}
 \country{USA}
  \postcode{63130}
}
\email{xuebing1234@gmail.com}

\author{Ahmed Sameh Said}
\affiliation{%
  \institution{School of Medicine\\Washington University in St. Louis}
 \city{St. Louis}
 \state{Missouri}
 \country{USA}
}
\email{said\_a@wustl.edu}

\author{Ziqi Xu}
\affiliation{%
  \institution{McKelvey School of Engineering\\Washington University in St. Louis}
 \city{St. Louis}
 \state{Missouri}
 \country{USA}
}
\email{ziqixu@wustl.edu}

\author{Hanyang Liu}
\affiliation{%
  \institution{McKelvey School of Engineering\\Washington University in St. Louis}
 \city{St. Louis}
 \state{Missouri}
 \country{USA}
}
\email{hanyang.liu@wustl.edu}

\author{Neel Shah}
\affiliation{%
  \institution{School of Medicine\\Washington University in St. Louis}
 \city{St. Louis}
 \state{Missouri}
 \country{USA}
}
\email{neel.shah@wustl.edu}

\author{Hanqing Yang}
\affiliation{%
  \institution{McKelvey School of Engineering\\Washington University in St. Louis}
 \city{St. Louis}
 \state{Missouri}
 \country{USA}
}
\email{alberty@wustl.edu}

\author{Philip Payne}
\affiliation{%
  \institution{School of Medicine\\Washington University in St. Louis}
 \city{St. Louis}
 \state{Missouri}
 \country{USA}
}
\email{prpayne@wustl.edu}

\author{Chenyang Lu}
\authornote{Corresponding author.}
\affiliation{%
  \institution{McKelvey School of Engineering \\ School of Medicine \\Washington University in St. Louis}
  \streetaddress{1 Brookings Drive}
 \city{St. Louis}
 \state{Missouri}
 \country{USA}
  \postcode{63130}
}
\email{lu@cse.wustl.edu}

\renewcommand{\shortauthors}{Bing Xue et al.}

\maketitle
\section{Introduction}

The severe acute respiratory syndrome coronavirus-2 (SARS-CoV-2) and the associated COVID-19 pandemic have created a substantial and unforeseen burden on the global healthcare system \cite{abrams2020ecmo}. With a global mortality of over 6.8 million (as of Feb 2023), there is considerable focus on therapeutic solutions for patients with most severe manifestations of the disease. In many cases, scarce treatment options need to be considered and evaluated to support patients' lives. In particular, World Health Organization recommends extracorporeal membrane oxygenation (ECMO) for patients who are refractory to conventional therapies, a support modality only available in expert centres with sufficient experience \cite{world2021covid}. As such treatments are technically complex and resource-intensive with difficulty in predicting outcomes, proper treatment evaluation and assignment for ECMO has been the subject of significant debate since the start of the pandemic \cite{falcoz2020extracorporeal,haiduc2020role}. 
Recent reports point to the vitalness of ECMO support over 14,000 reported COVID-19 patients with an overall hospital mortality of approximately 47\% \cite{ELSO}. In comparison, nearly 90\% who couldn't find a spot at an ECMO center died, and these patients were young and previously healthy, with a median age of 40 \cite{gannon2022association}. In addition to the vitalness, demand for ECMO far exceeded its availability leading to numerous patients waiting for ECMO support. To date, patient triage and ECMO resource allocation have been limited to the use of Intensive Care Unit (ICU) illness markers and markers of severe medically refractory respiratory failure, neither of which has been validated to predict patients who would ultimately benefit from this resource-intensive, high-risk therapy \cite{salyer2021first,shekar2020extracorporeal}. 
These gaps in knowledge highlight the need to develop clinically applicable predictive models to assist clinicians in identifying patients most likely to benefit from ECMO support and evaluating the treatment effect thus aiding in patient triage and the necessary resource allocation \cite{salyer2021first,huang2021role}.

From the perspective of treatment effect analysis, treatment assignment indicates whether a patient received the ECMO treatment, “factual outcomes” corresponds to patient’s discharge status (i.e. whether the patient survived/died), and patient’s features (or coviatates) are the electronic health records (EHR) before ECMO initiation. To support ECMO treatment decisions, we need to make two-side estimations. First, we need to model the probability of getting treatment for each patient (propensity scoring), to reflect the underlying treatment assignment policy \cite{rubin} as well as the associated risk consideration, as ECMO itself can lead patients to death. Second, we need to estimate the impact of each treatment decision, calculated by the survival/death difference with and without treatment.

Developing an ECMO decision-assistant model differs from typical supervised machine learning problems  or standard treatment effect problems in healthcare. Compared with a supervised clinical problem \cite{xue2021use,abraham2023integrating,jiao2022continuous,li2022self,shi2018obstructive,xue2022distance,xue2022validation,liu2022predicting}, the entire vector of treatment effects can never be obtained, but only the factual outcomes aligned with the individualized treatment assignments. Compared with a typical treatment effect problem \cite{spirtes2009tutorial,rubin}, it faces the challenges of strong selection bias, scarcity of treatment cases and curse of dimensionality. 
First, unlike randomized controlled trials or many common datasets, ECMO prediction is prone to strong selection bias. ECMO is only applied to high-risk patients with severe symptoms, when few life-supporting alternatives are available. As a result, the features characterizing severity are significantly different from those among non-ECMO-treated patients (referred hereafter as control patients). For treatment effect models that apply a supervised learning framework for each treatment option separately, the learned models would not generalize well to the entire population. Second, the technical complexity and resource intensiveness of such treatment limit the number of ECMO assignments, resulting in a much smaller cohort size when compared to control patients.  In fact, a retrospective study shows that fewer than 0.7\% of critically-illed COVID-19 patients received ECMO treatment\cite{bertini2021ecmo}. Moreover, the EHR dataset contains hundreds types of measurements/lab tests, and only a small subset of these measurements/tests would be conducted to each patient. Due to the limited cohort size in ECMO patients, it is challenging for most machine learning models to overcome the curse of dimensionality without falling into over-fitting. Instead of directly capturing the relationship between the prediction tasks and these partially-observed high-dimensional input features, a lower-dimensional representation of inputs is desired \cite{xuebingkdd}.  

In this paper, we tackle these challenges and propose \textit{Treatment Variational AutoEncoder (TVAE)}, a novel approach that uses a disentangled and balanced latent representation to infer a subject’s potential (factual and counterfactual) outcomes and treatment assignment. It leverages the recent advances in representation learning, and extends the capability of deep generative models in the following aspects:
\begin{itemize}
\item \textbf{Treatment Joint Inference}: The lower-dimensional latent representation of TVAE is \textit{semi-supervised} by treatment assignment and factual response. This architecture eliminates the need of auxiliary networks for prediction, and regulates the counterfactual prediction.  
\item \textbf{Distribution Balancing}: To generate an accurate latent representation, the selection bias is delicately disentangled from other latent dimensions to facilitate treatment assignment prediction and maximize the information sharing between two groups. 
\item \textbf{Label Balancing}: Utilizing the generative function of TVAE, we create \textit{fake} ECMO cases from the posterior distribution of \textit{real} ECMO cases, hence addressing the data imbalance and over-fitting without perturbing the patient distribution. 
\end{itemize}

  Our proposed TVAE outperforms state-of-the-art treatment effect models in predicting ECMO treatment assignment as well as factual responses (with or without ECMO), validated by two large real-world COVID-19 datasets consisting of an international dataset including 118,801 Intensive Care Units (ICU) patients from 1651 hospitals and a institutional dataset including 6,016 ICU patients from 15 hospitals. It also achieves the best performance in estimating individual treatment effects using the public IHDP dataset. While this work is motivated by and evaluated in the context of ECMO treatment, the proposed approach may be generalized for other treatment estimation tasks facing selection bias, label imbalance, and curse of dimensionality.

\section{Related Work}

\subsection{ECMO Treatment Prediction} 
To date, there exist a substantial gap between existing studies and ECMO treatment analysis. Most studies are limited to ECMO mortality scoring systems \cite{pappalardo2013predicting, schmidt2014predicting, schmidt2013preserve, shah2023validation} that rely on identifying pre-ECMO variables, which are available and validated only from patients already supported on ECMO. As such, none of these scores have utilized an appropriate matched non-ECMO cohort, rendering them incapable of identifying the  patients who should receive ECMO support \cite{shah2021extracorporeal,short2022extracorporeal}. A recent study uses gradient boosting trees to predict ECMO treatment assignment \cite{xue2022ecmo}, but it is not designed to predict treatment effect in terms of factual and counterfactual outcomes, which is an important contribution of this work. 

\subsection{Individual Treatment Effect} 
To predict each individual's treatment assignment and treatment effects, the most intuitive approach is to build separate single-task predictors for propensity (probability of receiving ECMO treatment), treatment outcome and control outcome (i.e., survival or death with/without ECMO treatment), respectively  \cite{alaa2017dcnpd,johansson2016bnn,shalit2017cfrnet,chipman2010bart}. Such models are generalized as meta-learners in recent literature \cite{curth2021meta, curth2021st}. These models are prone to the strong selection bias and label imbalance in ECMO treatment assignment, as each treatment-outcome predictor is only exposed to a specific patient group, but not the whole population.

 Another popular approach reported in the literature is by adapting non-parametric models for individual level treatment effect. The simplest version is the k nearest neighbors model, and more advanced models are adapted from tree-based ensemble models. Causal Forest has been developed from random forest to obtain a consistent estimator with semi-parametric asymptotic convergence rate \cite{wager2018cf}. Bayesian Additive Regression Trees (BART)-based methods have been proposed to build the trees with the regularization prior and the backfitting Markov chain Monte Carlo (MCMC) algorithm \cite{chipman2010bart, hahn2020bayesian}. Compared to neural network-based solutions, it is hard for such models to build and regularize the patients representations, hence the curse of dimensionality in characterizing scarce treatment group remains a challenge. 
 
 \subsubsection{Representation Learning} 
 Our proposed TVAE is related to earlier works using deep representation learning for treatment effect estimation. To build a balanced representation between treatment and control groups, various strategies are adopted to force information sharing between the groups. An popular strategy is to remove the selection bias from the representations of treatment and control groups, hence the representations of both groups are similar. Such representation can be from shared layers (such as SNet \cite{curth2021st}, DCN-PD \cite{alaa2017dcnpd}, Dragonnet \cite{shi2019adapting}, TARNET \cite{shalit2017tarnet}, BNN \cite{johansson2016bnn}) or separate networks (e.g., TNet \cite{curth2021st}). In contrast, TVAE acknowledges the strong selection bias in ECMO data hence does not force the representations of two groups to be the same. Instead, TVAE \textit{disentangles} the representation of patients into different aspects (or latent dimensions), and extracts the biased aspect to a designated latent dimension. By doing so, the remaining latent dimensions are naturally "balanced" as both groups share the similar information in these aspects. Considering the rareness of ECMO treatment events, TVAE further avoids the potential over-fitting by maximizing the cross-group similarities in these remaining latent dimensions. 
 
 \subsubsection{Deep Generative Models} 
 Another direction to estimate treatment effects is through deep generative models. CEVAE \cite{louizos2017cevae} and IntactVAE \cite{wu2021intact} adapt the Variational Autoencoder (VAE) to transform the representation of all patients into a common latent space, and build auxiliary networks to predict factual and counterfactual outcomes. GANITE, on the other hand, learn the counterfactual
distributions instead of conditional expected values \cite{yoon2018ganite}. 
As a VAE-based framework, TVAE possesses the salient properties of the abovementioned works, but in a different fashion. TVAE has a re-organized latent space that encodes patients' characteristics by explicitly expressing the predicted distributions of treatment assignment as well as factual and counterfactual outcomes. This eliminates the extra complexity of adding auxiliary networks and its potential insufficient training (as each auxiliary network for treatment outcome is only trained by a subset of data).  Such re-organized latent space further regulates the counterfactual outcomes through the clustering effect as well as input reconstruction. Moreover, TVAE leverages its generative feature to tackle the data imbalance in ECMO prediction: it augments ECMO cases by upsampling from their latent distributions, hence achieving label balance in training iterations.

\begin{figure*}[tbh!] 
    \centering
    \includegraphics[width=0.9\textwidth]{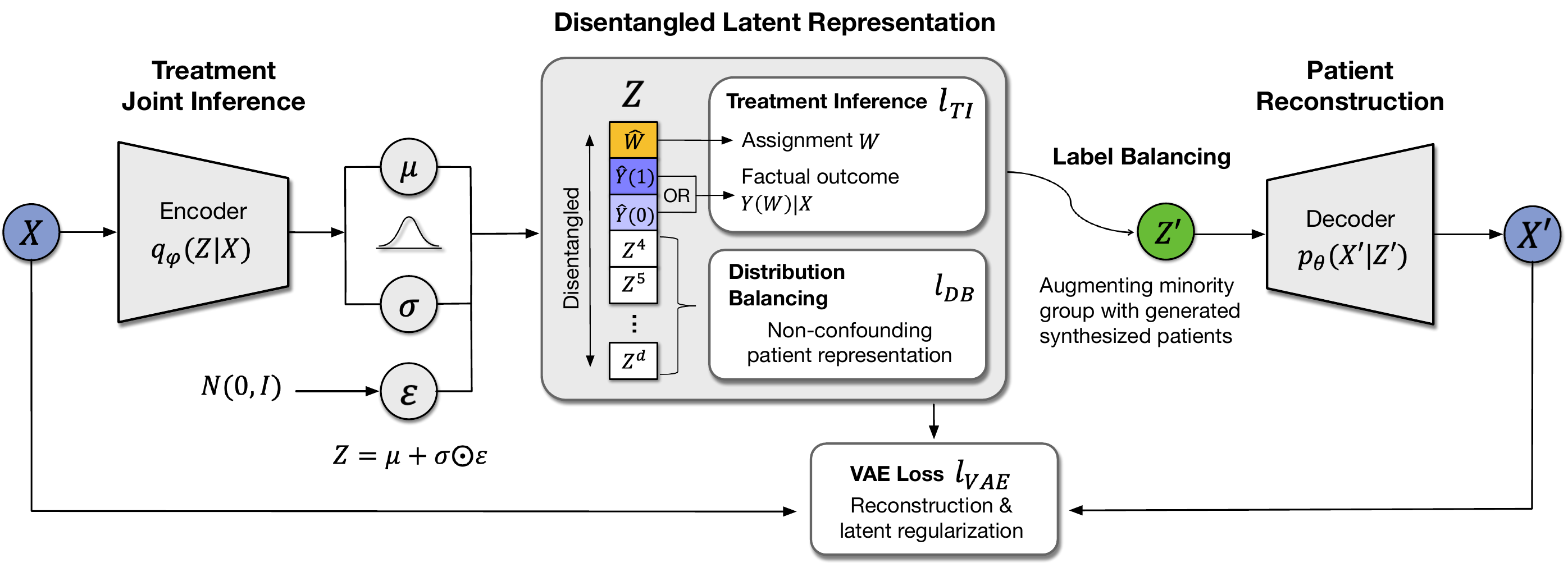}
    \caption{Overview of TVAE. 1)  the Treatment Joint Inference encodes input $X$ into a latent space with direct prediction of treatment assignment and potential outcomes; 2) Distribution Balancing module disentangles the latent space and balances the nonconfounding latent dimensions between treatment and control groups; 3) the Label Balancing module upsamples the under-represented minority group from the learned latent distribution; 4) the upsampled \textit{fake} data are merged with the original data as the updated training inputs for the next iteration. }
\end{figure*}

\section{Problem Formulation}
Throughout this paper, we adopt Rubin’s potential outcomes model \cite{rubin} and consider the population of COVID-19 subjects where each subject $i$ is associated with a $p$-dimensional feature $X_i \in X \subseteq \mathbb{R}^p$, a binary treatment assignment indicator  $W_i \in \{0,1\}$, and two potential outcomes $Y_i(1)$, $Y_i(0)\in \{0,1\}$ drawn from a Bernoulli distribution $
    (Y_i(1), Y_i(0))|X_i \sim P(.|X_i)
$. For an observational dataset $D$ comprising $n$ independent samples of
the tuple $\{X_i, W_i, Y_i(W_i)\}$, where $Y_i(W_i)$ and $Y_i(1-W_i)$ are the factual and the counterfactual outcomes, respectively, we are interested in the probability of treatment assignment (propensity score) $p(x) = P(W_i = 1|X_i)$, the potential outcome with treatment $\mathbb{E} [Y_i(1) |X_i ]$ and the potential outcome without treatment $\mathbb{E} [Y_i(0) |X_i ]$. As the treatment outcomes are binary (survival or death), the treatment is \textit{"impactful"} only if treatment leads to a change from death to survival. In a more generalized setting, a \textit{proxy} of treatment impact is always used, which is the reduction in mortality risk (the individualized treatment effect, or ITE) $ T(X_i) = \mathbb{E} [Y_i(1) - Y_i(0) |X_i ]$ \cite{rubin}. 
 As the counterfactual outcome, $Y_i(1-W_i)$, can never be observed in practice, direct test-set evaluation of the treatment effect is impossible. Existing counterfactual estimation methods usually make the following important assumptions:

\begin{assumption}[No Unmeasured Confounding]
Given $X$, the outcome variables $Y_0$ and $Y_1$
are independent of treatment assignment, i.e., ($Y_0$, $Y_1$) $\ind W |X$.
\end{assumption}

\begin{assumption}[Positivity]
For any covariates $X_i$, the probability to receive/not receive treatment is positive, i.e., $0 < P(W = w|X = X_i)$ $< 1, \forall w$ and $i$.
\end{assumption}

The first assumption comes from the fact that every candidate patient is continuously measured in various aspects that might be relevant to treatment assignment and potential treatment outcomes, and clinicians rely on these measures to make reasonable treatment decisions. As all such measurements/tests are captured in the EHR dataset, we believe the dataset has included all confounding variables. The second assumption holds because only potential treatment candidates are included in this study, and no patients \textit{must} receive ECMO treatment in real clinical scenario. 
More discussions on two assumptions are attached in Appendix A.

\section{TVAE}

Our proposed TVAE framework builds upon a VAE architecture to transform the high-dimensional inputs into a lower-dimensional latent representation. A VAE jointly trains a decoder network (parameterized with $\theta$) with an encoder (parameterized with $\phi$) to recover the original inputs $X$ from the latent encoding $Z$ while regularizing the learned latent space to be close to the prior distribution. For a vanilla VAE, the loss function of training the encoder-decoder network can be written as:
\begin{equation}
\label{eq:VAE}
    \begin{split}
    l_{\text{VAE}}(\phi,\theta) &= \sum_{X_i \in \mathcal{X}} -\mathbb{E}_{Z_i\sim q_{\phi}(Z_i|X_i) } \Bigl[ \log p_{\theta} (X_i|Z_i)\Bigr]  \\
    &+ KL\Bigl(q_{\phi}(Z|X) || p(Z)\Bigr) 
    \end{split}
\end{equation}
where $p_{\phi}(Z|X)$ and $q_{\theta}(X|Z)$ are the learned approximation of the posterior and likelihood distributions, and $p(Z)$ is the prior assumption. This loss function consists of two parts: a reconstruction term and a Kullback–Leibler (KL) divergence regularizer. The former loss maximizes the recovery of the inputs, hence the latent encoding must be a  truthful representation of patients. The latter helps learn an approximation to the true underlying characteristics of the patient data and produce a compact, smooth and meaningful latent space, which can make the learned latent representations easier to use for downstream tasks such as clustering and data generation~\cite{xuebingkdd}.   

In our proposed TVAE, we design a customized encoder, Treatment Joint Inference (Section \ref{sec:TI}), to simultaneously make inferences on both treatment assignment decision and treatment effect (by predicting both factual and counterfactual outcomes).
We further propose two other schemes for TVAE, Distribution Balancing (Section \ref{sec:DB}) and Label Balancing (Section \ref{sec:LB}), to tackle the challenges of selection bias in treatment effect estimation and label imbalance in ECMO assignment prediction, respectively.
Figure 1 shows an overview of our proposed TVAE framework.

\subsection{Treatment Joint Inference}
\label{sec:TI}
Unlike the previous ITE approaches~\cite{curth2021st,shalit2017cfrnet,johansson2016bnn,alaa2017dcnpd,shi2019adapting}, we aim to simultaneously make inferences on treatment assignment and treatment effect (by predicting both factual and counterfactual outcomes) based on a compact neural structure without any auxiliary predictor network. In TVAE, this is achieved by assigning specific latent dimensions in the latent representation $Z$ as the estimate of the treatment assignment $W$, and the observed treatment outcome $Y(1)$ or control outcome $Y(0)$.
Intuitively, irrelevant to the actual assignment (Assumption 1), both the factual and counterfactual treatment outcomes are true reflection of the patient physical status, and therefore can be naturally considered as part of the patient representation $Z$.

Given the latent representation produced by the variational encoder, $Z=(Z^1, Z^2, ..., Z^d)\in\mathbb{R}^d$, let the first three dimensions, $Z^1, Z^2$ and $Z^3$, be encoded to estimate the treatment assginment (i.e., $Z^1=\hat{W}$), treatment outcome (i.e., $Z^2=\hat{Y}(1)$), and control outcome (i.e., $Z^3=\hat{Y}(0)$), respectively
. Then the estimated assignment $\hat{W}=Z^1$ and the factual outcome $\hat{Y}(W)=Z^{3-W}$ (where $W\in\{0,1\}$) is supervised by minimizing the following loss:
\begin{align}
\label{eq:TI}
    l_{\text{TI}}(\phi,\theta) = \sum_{X_i \in \mathcal{X}} -\mathbb{E}_{Z_i^1, Z_i^{3-W}\sim q_{\phi}} \Bigl[
    \log p(W_i,Y_i(W_i)|Z_i^1, Z_i^{3-W})\Bigr]
\end{align}
where $W_i$ and $Y(W_i)$ are the true treatment assignment and factual outcome.
For binary treatment outcomes, such as survival/death for ECMO data, cross entropy is used to implement Eq. (\ref{eq:TI}). 
For continuous treatment outcomes, such as the semi-simulated IHDP dataset (Section 5.4), mean square error is used.

The joint inference encoding allows us to make use of the label information (factual outcomes and treatment assignment) to optimize the encoder.
Similar idea of incorporating supervised information into VAE can be found in prior work such as conditional VAE (CVAE)~\cite{NIPS2015_8d55a249} and CEVAE~\cite{louizos2017cevae}.
However, they all rely on an auxiliary network for label prediction. Moreover, as part of the latent representation, the label inference is also directly regularized by the unsupervised data reconstruction.

Due to the absence of counterfactual ground truth $Y(1-W)$, it's impossible to directly learn the encoded dimension $Z^{3-(1-W)}$ in a supervised fashion as in Eq. (\ref{eq:TI}) above. However, with our TVAE,
the estimate of counterfactual outcomes can be optimized through a semi-supervised process by jointly optimizing $l_{\text{VAE}}(\phi,\theta)$ in Eq. (\ref{eq:VAE}) and $l_{\text{TI}}(\phi,\theta)$ in Eq. (\ref{eq:TI}). 
On one hand, the learning of $\hat{Y}(1-W)$ is regularized by the patient reconstruction process in $l_{\text{VAE}}(\phi,\theta)$. To minimize reconstruction error, $Z^{3-(1-W)}$  must be a truthful representation of $X$.
On the other hand, the semi-supervised setting of TVAE helps the latent encoding \textit{"share"} outcome information across similar patients.
More concretely, similar patients are clustered close to each other in a well-learned smooth and compact latent space of VAE, the counterfactual outcome of a patient (e.g., treatment outcome of a control patient) can be inferred using the factual outcomes of similar patients with a different treatment assignment in the latent neighborhood . 

\subsection{Distribution Balancing}
\label{sec:DB}
\subsubsection{Disentangled Latent Representation}
Previous studies maximize the distribution similarity in control and treatment groups, therefore selection bias is removed and the representation of the scarce treatment group can be regularized \cite{johansson2016bnn,yao2018site,shalit2017cfrnet} by the control group. However, 
instead of removing the selection bias from patients' representations, we argue that it should be utilized to enhance propensity scoring. This can be achieved by disentanglement \cite{chen2018isolating,xuebingkdd} together with the Treatment Joint Inference. By enforcing disentanglement in the latent space while jointly encoding the predicted outcomes ($W$, $Y(W)$) in the designated latent dimensions, the selection bias flows into $Z^1$ for treatment assignment prediction, and naturally the remaining latent dimensions are balanced between treatment and control group. To show this, consider the disentanglement through mini \textit{Total Correlation} (TC) in the $d$-dimensional latent space ~\cite{chen2018isolating},
where the TC of the set of variables, $Z^{1:d}=\{Z^1, Z^2, ..., Z^d\}$, is defined as the ratio between the joint distribution and the product of the marginals, i.e., 
$TC(Z^{1:d}|X) := KL\bigl(q_{{\phi}}(Z|X) \bigl|\bigr|  \prod_{j=1}^{d} q_{{\phi}}(Z^j|X)\bigr)$.

\begin{proposition}
    Given $Z^1 = \hat{W}$, conditioned on the data $X$, the total correlation of $Z^{1:d}$ equals the sum of the total correlation of $Z^{2:d}$ and the mutual information between the first dimension $\hat{W}$ and all the other dimensions $Z^{2:d}$, i.e., 
\begin{equation}
        TC(Z^{1:d}|X) = TC(Z^{2:d}|X) 
    + I(\hat{W}|X;Z^{2:d}|X)
\end{equation}
where $I(A;B)$ is the mutual information between variable $A$ and $B$.
\end{proposition}

The proof can be found in Appendix C. As the Joint Encoding forces $Z^1$ to approximate $W$, clearly TC is minimized when $I(\hat{W};Z^{2:d}|X)$ is minimized. Alternatively, we can see this from the fact that all terms on both sides are nonnegative, hence the minimization is reached only if the mutual information between $Z^1$ and other latent dimensions is 0. 
Hence, the remaining latent dimensions neither contain any treatment assignment information, nor are they affected by the selection bias. Due to the curse of dimensionality, the latent representation of treatment group using a deep encoder easily overfits and become poorly generalizable. Note that, however, we can use the learned representation in the control group to \textit{guide} the learned representation in the treatment group. The idea is as follows: when a latent dimension only preserves the non-confounding information, then the distribution is irrelevant to treatment assignment. For example, as gender is not considered in treatment assignment, the latent encoding for gender should be distributed indifferently in treatment and no-treatment groups. This can be expressed as
\begin{proposition} For any dimension
$ Z^j$ in the latent representation $ Z:=[Z^1,Z^2,...,Z^d]$,
$ Dist(Z^j(0))=Dist(Z^j(1))$, if and only if $
    Z^j\ind (Y(0),Y(1),W)|X 
$, where $Dist(.)$ denotes the distribution. 
\end{proposition}

\subsubsection{Distribution Matching}
This proposition provides the ground to further regulate the remaining latent dimensions by minimizing the distribution difference.  
An intuitive way is to calculate the maximum mean discrepancy (MMD) for the distance on the space of probability measures, as it has an unbiased U-statistic estimator, which can be used
in conjunction with gradient descent-based methods \cite{tolstikhin2017wasserstein,tolstikhin2016mmd}.  Considering the complexity of potential distributions, we apply the kernel trick so that the MMD is zero if and only if the distributions are identical in the projected Hilbert space. Denote $q_{\phi}(Z^j|X(0))$ by $P^j_{-}$ and $q_{\phi}(Z^j|X(1))$ by $P^j_{+}$, the kernelized MMD metric can be expressed as:
\begin{equation}
    \begin{split}
    &\text{MMD}\Bigl(q_{\phi}\bigl(Z^j|X(0)\bigr),q_{\phi}\bigl(Z^j|X(1)\bigr)\Bigr) \\
    =& \ 
    \text{MMD}\bigl(P^j_{-},P^j_{+}\bigr)\\
    =& \ 
    \Bigl\|
    \int_Z k_Z(z,.)dP^j_{-}(z)  -
    \int_Z k_Z(z,.)dP^j_{+}(z)
    \Bigr\|_{\mathcal{H}_k}
    \end{split}
\end{equation}
where $k$ is an infinite-dimensional radial basis function (RBF) kernel and $\mathcal{H}_k$ is the corresponding reproducing kernel Hilbert space. Alternative distribution matching methods, such as linear MMD without kernelization, Wasserstein Distance and KL divergence are evaluated in treatment effect estimation in Sec. 5.5.

With balanced representation, two factors are inserted into the loss function for joint-optimization: the TC loss and the MMD loss:
\begin{equation}
    \begin{split}
    l_{\text{DB}}(\phi,\theta) = TC(Z^{1:d}|X) 
    + \gamma\sum_{j=4}^{d}\text{MMD}\Bigl(q_{\phi}\bigl(Z^j|X(0)\bigr),q_{\phi}\bigl(Z^j|X(1)\bigr)\Bigr)
    \end{split}
\end{equation}
where hyperparameter $\gamma$ is used to adjust the scales of the MMD loss to be similar to the TC loss.

\subsection{Label Balancing}
\label{sec:LB}
The scarcity of ECMO treatment resources and the clinical considerations in treatment assignment induce the selection bias in COVID-19 patients. The limited number of the treatment assignment leads to significant data imbalance (less than 3\% in ECMO datasets), thus an encoder network easily ignores the minority group, overfits the minority group or underfits the majority group in latent representation. 

To address this issue, we enrich the training data by augmenting the under-represented patients while maintaining their intrinsic characeristics. Instead of using extra data augmentation models \cite{sandfort2019gan,wang2020deepgenerative} to generate \textit{fake} patients, we utilize the generative power in our established latent space to sample more ECMO cases, as the latent representation of TVAE (or any VAE-based model) is a posterior distribution (as shown in Figure 1). Since the patients are sampled from the latent distribution of real ECMO cases, the generated "\textit{fake}" patients (constructed by passing the upsampled latent representations through decoder) do not change the distribution and characteristics of treatment group. Subsequently, the upsampled "\textit{fake}" data is concatenated with the original training data to form a more balanced inputs for model learning. 

For a well-balanced dataset, the loss function of TVAE is simply $
l_{\text{total}}(\phi,\theta) = l_{\text{VAE}}(\phi,\theta) + \alpha \cdot l_{\text{TI}}(\phi,\theta) + \beta \cdot l_{\text{DB}}(\phi,\theta)$, where hyperparameters $\alpha$ and $\beta$ are used to adjust each loss term to be in similar scales in the current dataset. Take the public IHDP dataset for example, we set $\alpha$ and $\beta$ to be 1 and 0.1, respectively. 
In the ECMO case where treatment cases are rare,  the label balancing module kicks in during the training iteration, hence the model input is the augmented dataset $X' \sim p_\theta(X'|Z')$ where $Z'$ is the up-sampled latent representations. The loss function can be expressed as:
\begin{equation}
    \begin{split}
    l_{\text{total}}(\phi,\theta) =& \sum_{X_i \in \mathcal{X}} -\mathbb{E}_{Z_i\sim q_{\phi}(Z_i|X_i) } \Bigl[\log p_{\theta} (X_i|Z_i) +
    \alpha\log p(W_i,Y_i|Z_i) \Bigr] \\
    &+ KL\bigl(q_{\phi}(Z|X') \| p(Z)\bigr) \\
    &+ \beta\cdot KL\Bigl(q_{{\phi}}\bigl(Z|X'\bigr)\big\|\prod_{j=1}^{d} q_{{\phi}}\bigl(Z^j|X'\bigr)\Bigr) \\
    &+ \beta\gamma\cdot\sum_{j=4}^{d}\text{MMD}\Bigl(q_{\phi}\bigl(Z^j|X'(0)\bigr),q_{\phi}\bigl(Z^j|X'(1)\bigr)\Bigr)
    \end{split}
\end{equation}
When the loss function converges, the latent encoding will not change by the added \textit{fake} patients, hence $ q_{\phi}(Z|X') =  q_{\phi}(Z|X)$.

\section{Experiments}
The first part of the experiments examine how TVAE performs under real ECMO settings. For each individual case, we look at the predicted treatment assignment and the factual response (mortality or survival) with the assigned treatment option. Since the factual  response may come from either treatment group or control group, the performance in response prediction suggests the overall performance in predicting both treatment and control (no-treatment) outcomes. 
To evaluate TVAE's performance in estimating treatment effect, we rely on the public synthesized datasets (where counterfactual outcomes are also available). Our experiments are designed to to answer the following questions: 
\begin{enumerate}
    \item Can TVAE predict the treatment assignment and factual responses? 
    \item How do the tailored components (DB and LB) contribute to the TVAE model? 
    \item How does TVAE perform in individual treatment effect estimation on synthesized datasets?     
\end{enumerate}

\begin{table}[t]
  \caption{Performance of Predicting Treatment Assignment and Factual Outcome (Response) on ISARIC Dataset.}
  \label{Table1}
  \centering
\fontsize{7.5pt}{9pt}\selectfont
  
  \begin{tabular}{lcccc}
    \toprule
    & \multicolumn{2}{c}{AUPRC}&\multicolumn{2}{c}{AUROC}\\
    \textsc{Model} & \textsc{Assignment} & \textsc{Response} &  \textsc{Assignment} & \textsc{Response}\\
    \midrule
    OLS & .1335 ± .0081  & .6225 ± .0027    & .8542 ± .0080  & .7149 ± .0013  \\
    KNN  & .1383 ± .0070  & .5693 ± .0025    & .5580 ± .0025  & .6743 ± .0018 \\
    BART & .2243 ± .0126  & .6571 ± .0126    & .9342 ± .0028  & .7489 ± .0013  \\
    CF  & .2748 ± .0152  & .6428 ± .0025    & .8865 ± .0072  & .7423 ± .0013  \\
    BNN  & N/A  & .5670 ± .0264    &  N/A &  .6890 ± .0181 \\
    DCNPD   & .2626 ± .0213  & .5438 ± .0201    & .9188 ± .0066  & .6208 ± .0101  \\
    TNet    & .0596 ± .0048  & .5931 ± .0075    & .8589 ± .0138  & .6983 ± .0077  \\
    SNet    & .0652 ± .0074  & .5948 ± .0023    & .8463± .0060  & .7043 ± .0020  \\
    TARNET    & .1350 ± .0089  & .6088 ± .0039    & .8650 ± .0039 & .7208 ± .0024  \\
    GANITE   & N/A  & .6524 ± .0209    & N/A  & .7488 ± .0125  \\
    CEVAE   & .2580 ± .0096  & .6415 ± .0193    & .9188 ± .0058  & .7293 ± .0109  \\
    Dragonnet    & .1439 ± .0121  & .6525 ± .0030    & .9069 ± .0060  & .7475 ± .0023  \\
    \midrule
    TVAE\textsuperscript{-DB} &  \underline{.2810 ± .0140} & \underline{.6669 ± .0031} &	\underline{.9306 ± .0040} &	\underline{.7605 ± .0008}
     \\
    TVAE\textsuperscript{-LB} &  .2226 ± .0185 & .6566 ± .0026 &	.9295 ± .0044 &	.7553 ± .0008
     \\      
    \textbf{TVAE}    & \textbf{.2970 ± .0195}  & \textbf{.6678 ± .0022}    & \textbf{.9431 ± .0032}  & \textbf{.7610 ± .0011}  \\
    \bottomrule
  \end{tabular}
\end{table}

\subsection{Data}
In this study, we constructed two real COVID-19 datasets from different continents. Data access agreement and IRB approval were acquired prior to the study, as shown in Appendix E, together with data processing pipeline, feature extraction methods, and characteristics of the cohort. Evaluating the individual treatment effect (ITE) estimation on the ISARIC and BJC ECMO data is impossible, since the ground truth of the counterfactual outcomes is not available in reality. Therefore, we also implement TVAE using the synthesized Infant Health and Development Program (IHDP) dataset, described in previous studies \cite{louizos2017cevae,chipman2010bart,wager2018cf,shalit2017cfrnet,shi2019adapting}. It consists of 1000 replications. In each replication, the dataset contains 747 subjects (139 treated and
608 control), represented by 25 covariates.

\paragraph{ISARIC Dataset}
The first ECMO dataset includes COVID-19 individuals from the International Severe Acute Respiratory and Emerging Infection Consortium (ISARIC)–World Health Organization (WHO) Clinical Characterisation Protocol (CCP), referred hereafter as ISARIC Data. Through international collaborative efforts, it covers 1651 hospitals across 63 countries from 26 January 2020 to 20 September 2021. We include a total of 118,801 patients who were admitted to an Intensive Care Unit (ICU) for at least 24 hours so that ECMO treatment is a feasible treatment option (hence the assumption of a positive probability for treatment assignment holds). Among these patients, 1,451 (1.22\%) received ECMO treatment. As the patients are from different hospitals with different treatment decision criteria, the characteristics in both treatment  and control groups are heterogenous. The mortality ratio is 40.00\% in treatment group and 50.56\% in control group. 

\paragraph{BJC Dataset}
The second ECMO dataset is a single institutional dataset, containing electronic health records (EHR) spanning 15 hospitals in Barnes Jewish HealthCare system. This dataset is referred hereafter as BJC Dataset. It contains COVID-19 patients admitted to ICU during 19 months (March 3rd 2020 - October 1st 2021). Among the total of 6,016 included patients, 134 (2.23\%) received ECMO treatment. As the patients are from the same healthcare system, the treatment assignment is made by a panel of clinical experts with consistent decision criteria. The mortality ratio in the treatment group is 47.01\% and in the control group is 18.99\%. Detailed data processing and feature extraction are provided in Appendix E.

\begin{figure}
    \centering
    \includegraphics[width=1.\linewidth]{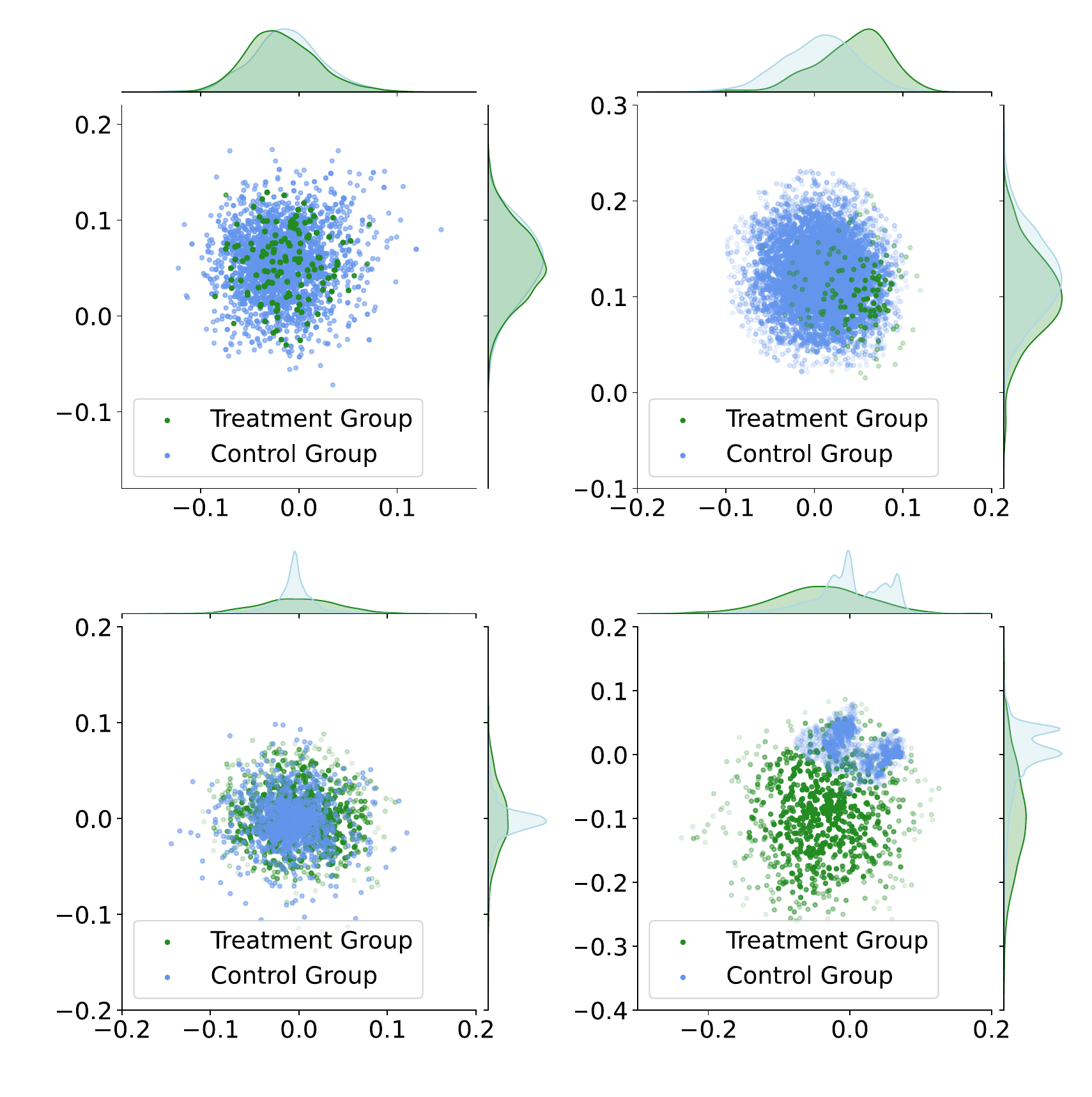}
    \captionsetup{labelformat = empty } 
    \caption{Figure 2: Latent Distributions of treatment and control groups in unsupervised dimensions with and without Distribution Balancing module. Left: with Distribution Balancing; Right: without Distribution Balancing; Top: BJC Dataset; bottom: ISARIC Dataset.}
    \label{fig:my_label}
\end{figure}
\begin{figure}
    \centering
\includegraphics[width=0.6\linewidth]{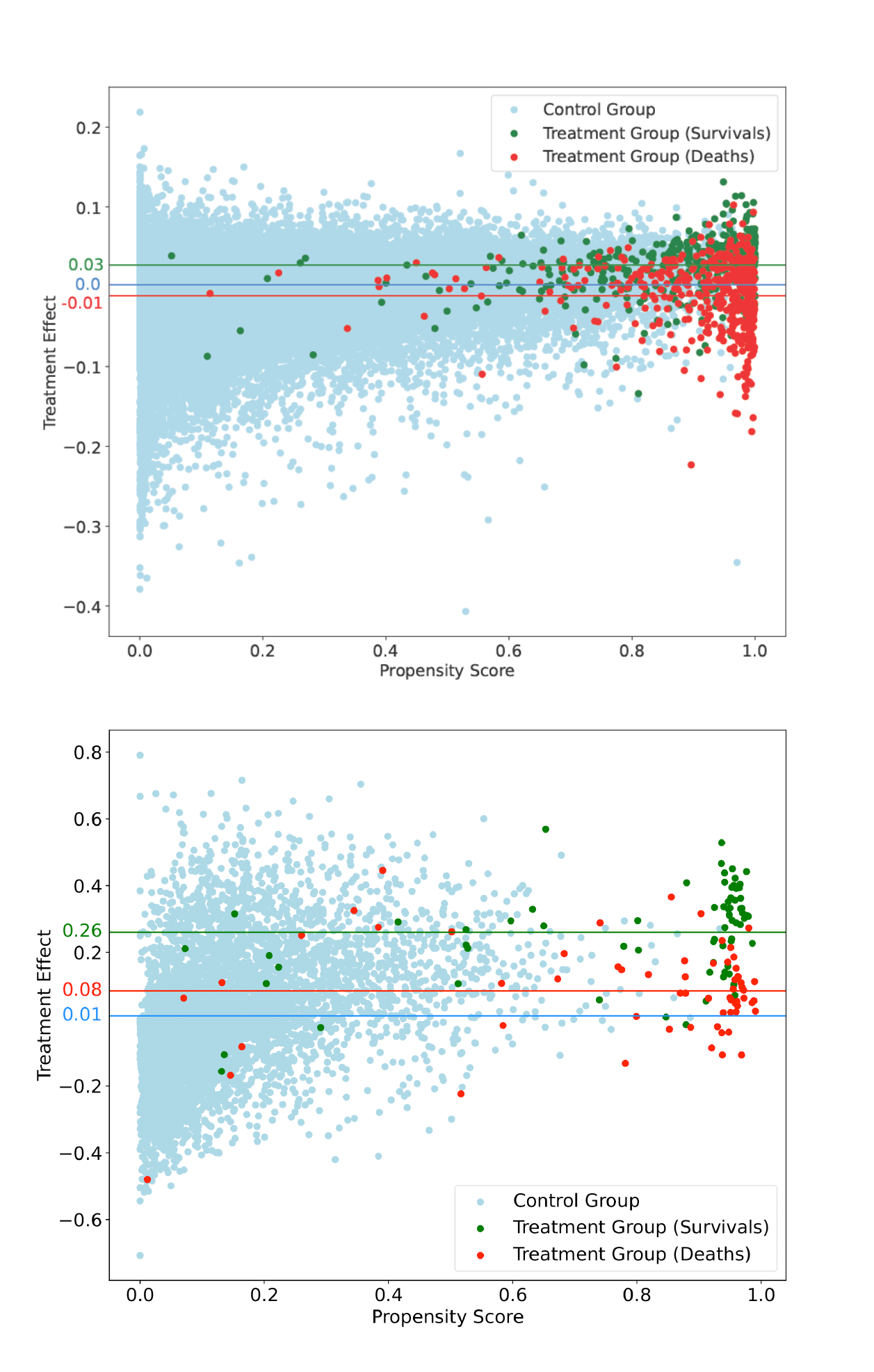}
\captionsetup{labelformat = empty } 
\caption{Figure 3: Risk stratification by propensity score (X axis) and treatment effects (Y axis) on control cases, died treatment cases and survived treatment cases. Horizon lines are the average treatment effects of each group. Top: ISARIC; bottom: BJC Dataset.}
    \label{fig:my_label}
\end{figure}

\subsection{Baseline Settings and Evaluation Metrics}
We implemented the state-of-the-art treatment effect algorithms that were discussed in the related works. This includes a linear ordinary least square model (OLS), a nonparametric k-Nearest Neighbor model (kNN), the tree-based causal inference models (BART \cite{chipman2010bart} and Causal Forest (CF) \cite{wager2018cf}), deep generative causal inference models (CEVAE and GANITE \cite{louizos2017cevae,yoon2018ganite}), and other deep representation causal inference models (DCN-PD, BNN, TNet, SNet, TARNET, and Dragonnet). \cite{alaa2017dcnpd,johansson2016bnn,yao2018site,shalit2017tarnet,shi2019adapting, curth2021st}.

Among all the models, we performed the grid-search of hyper-parameters. The final hyper-parameter settings are described in Appendix G. 
An ablation study is conducted to  investigate if the tailored components (Distribution Balancing and Label Balancing) help with ECMO prediction. This involves both quantitative evaluations by prediction performance and qualitative evaluation by visualizing the latent distributions. TVAE\textsuperscript{-DB} represents the model when distribution balancing is removed from the loss function, and TVAE\textsuperscript{-LB} represents the model when the minority cases are not up-sampled from the learned latent distribution. 

In ISARIC Dataset and BJC Dataset, both treatment assignment and factual outcomes are binary, therefore the area under the Receiver-Operating Characteristic curves (AUROC) is used to evaluate the predictive power of each model. Due to the scarcity of the treatment resources, we are interested in the trade-off of precision and recall, hence we also calculate the area under the Precision Recall curve (AUPRC). Quantitative results are reported with mean and standard error after 5-fold cross validation, where the stratification is random while preserving the treatment assignment ratio in each fold of patients.
For IHDP dataset, we followed the train/test split strategy with 1000 replications provided in the previous study \cite{louizos2017cevae}, and calculate the square root of the Precision in Estimation of Heterogeneous Effect (rPEHE), defined as:
$
    \epsilon_{rPEHE} = \sqrt{\frac{1}{N}\sum_{i=1}^{N} (\Tilde{Y_i(1)}-\Tilde{Y_i(0)} - (Y_i(1) -Y_i(0) ) )^2} 
$
where $\Tilde{Y_i(1)}$ and $\Tilde{Y_i(0)}$ are the estimated treatment and no-treatment outcomes, respectively. 

\begin{table}[t]
  \caption{Performance of Predicting Treatment Assignment and Factual Outcome (Response) on BJC Dataset.}
  \label{Table2}
  \centering
\fontsize{7.5pt}{9pt}\selectfont
  
  \begin{tabular}{lcccc}
    \toprule
    & \multicolumn{2}{c}{AUPRC}&\multicolumn{2}{c}{AUROC}\\
    \textsc{Model} & \textsc{Assignment} & \textsc{Response} &  \textsc{Assignment} & \textsc{Response}\\
    \midrule
    OLS & .6588 ± .0388  & .7530 ± .0114    & .9542 ± .0057  & .9079 ± .0054  \\
    KNN  & .5049 ± .0361  & .6819 ± .0065    & .7306 ± .0280  & .8567 ± .0057 \\
    BART & .5657 ± .0198  & .7314 ± .0125    & .9561 ± .0085  & .9013 ± .0047  \\
    CF  & .6624 ± .0265  & .7750 ± .0058    & .9447 ± .0112  & .9227 ± .0112  \\
    BNN  & N/A  & .7023 ± .0192    &  N/A & .8887 ± .0072 \\
    DCNPD   & .6136 ± .0295  & .6380 ± .0150    & .9374 ± .0165  & .8609 ± .0068  \\
    TNet    & .5403 ± .0641  & .7034 ± .0495    & .9046 ± .0169  & .8612 ± .0228  \\
    SNet    & .4977 ± .0411  & .6647 ± .0096    & .8574 ± .0162  & .8366 ± .0110  \\
    TARNET    & .5715 ± .0598  & .7192 ± .0100    & .9567 ± .0079 & .8814 ± .0070  \\
    Dragonnet    & .6522 ± .0394  & .7212 ± .0085    & .9551 ± .0100  & .8834 ± .0056  \\
    GANITE   & N/A  & .6991 ± .0256    & N/A  & .8930 ± .0152  \\
    CEVAE   & .7007 ± .0070  & .3937 ± .0104    & .9469 ± .0108  & .7349 ± .0062  \\
    \midrule
    TVAE\textsuperscript{-DB} &  \underline{.6857 ± .0976} & .7568 ± .0174 &	.9488 ± .0135 &	\underline{.9148 ± .0050}
     \\
    TVAE\textsuperscript{-LB} &  .6767 ± .0434 & \underline{.7672 ± .0087} &	\underline{.9509 ± .0108} &	.9127 ± .0059
    \\
    \textbf{TVAE}    & \textbf{.7344 ± .0357}  & \textbf{.7856 ± .0083}    & \textbf{.9582 ± .0123}  & \textbf{.9243 ± .0035}  \\
    \bottomrule
  \end{tabular}
\end{table}
\subsection{Quantitative and Qualitative Results in ECMO Prediction}
The quantitative results of prediction metrics are reported in Table 1 for ISARIC Dataset, and Table 2 for BJC Dataset. For models that do not (explicitly) predict propensity score, the corresponding fields are labeled as "N/A".
\subsubsection{Overall performance}
Our proposed TVAE outperforms all the baseline methods in ISARIC and BJC datasets by all metrics. Such observation has statistical significance, measured by 90\% confidence interval and paired one-tail t-test. 
\subsubsection{Model complexity v.s. over-fitting} 
On BJC Dataset where the cohort size of ECMO patients is significantly small (134 cases) and the input feature dimensions are relatively large (178 features), simple linear models such as OLS perform better than complex tree-based models and most of the deep-learning models. This is likely the result of over-fitting when complex models try to characterize scarce treatment cases. On the other hand, more complex models (such as BART, CF, GANITE, CEVAE and Dragonnet) outperform OLS on ISARIC Dataset, as it is significant larger and more heterogeneous. Armed with multiple novel regularization schemes (semi-supervised encoding, disentanglement and distribution matching), TVAE demonstrates its robustness in BJC cohort, meanwhile its deep learning architecture is capable of discovering the underlying nonlinear patterns in ISARIC cohort. 
\subsubsection{Tree-based learning v.s. deep representation learning}
Similar as observed in a previous study \cite{grinsztajn2022tabular,xue2021use}, tree-based models have high prediction performance in tabular data, but performance deteriorates when facing label imbalance. In the label-imbalanced treatment assignment prediction, CEVAE and DCNPD achieves similar or better AUROC/AURPC than tree-based models. In our TVAE, the carefully designed architecture improves its representation learning, achieving even more significant improvement over tree-based models in imbalanced problems (treatment assignment prediction) while matching the performance in factual prediction. 

\subsubsection{Ablation analysis.} 
We are interested in how the dedicated components affect the model performance. To quantitatively evaluate the contribution of Distribution Balancing, we remove the MMD loss and disentanglement, and reduce the weight of KL divergence (KL divergence helps disentanglement \cite{higgins2017beta}, but it must be kept for clustering effects). After the removal of the Distribution Balancing, all metrics dropped, and the decrease is more significant in the BJC Dataset, resulting in 6.6\% reduction in the AUPRC of treatment assignment prediction. This is consistent with the fact that the BJC Dataset has limited treatment samples, hence more prone to overfitting. To qualitatively visualize how Distribution Balancing affects the latent representations, in Fig. 2 we plot the direct visualizations of projected inputs in the latent space (as well as probability density distribution) of treatment and control groups in unsupervised dimensions, with and without Distribution Balancing. Each scatter point represents an ECMO/control case, and the transparency is proportional to the empirical probability density at this location. The empirical probability density is calculated through kernel density estimation (KDE). We used the 4th and 5th latent dimension, as the first three dimensions are for treatment assignment, factual and counter factual estimation, respectively. A clear divergence between treatment and control groups is observed after the removal of Distribution Balancing, suggesting the potential existence of selection bias in these dimensions or overfitting. 

To quantitatively evaluate the contribution of Label Balancing, we remove the upsampling from the training process. Without balancing the treatment/control groups, we observed 25.1\% and 7.9\% drops in the AUPRC of treatment assignment on ISARIC Dataset and BJC Dataset. To investigate the optimal upsampling ratio,  we vary the relative size of the upsampled ECMO cases and the optimal unsampled size is found to be 80\% (relative to the number of control cases) for ISARIC Dataset and 60\% for BJC Dataset, as can be seen in Appendix H. 
Note that the ablation of Treatment Joint Inference is not within the scope of this study, since removing it will totally change the architecture of the proposed work. 

\subsubsection{Risk stratification.} Given TVAE as a decision assistance tool, we want to visualize the consistency between risk predictions and observed outcomes. In Fig. 3, we plot the predicted propensity score, and the treatment effects (measured by the mortality risk with ECMO minus the mortality risk without ECMO). Via propensity scoring, TVAE separates cases that are more likely to be assigned treatment from cases that are not, and the stratification matches with the actual clinical decisions. By predicting the potential risk reduction by ECMO treatment, our model divides the treatment group into those who benefit more from ECMO and those who benefit less. The division is consistent with the actual death and survival outcomes.

\subsection{Individual Treatment Effect Estimation}

We first compare the treatment effect estimation between TVAE and the state-of-the-art algorithms. Recent literature has noted the inconsistency of results reported in existing literature, where the calculated metrics might be different, or sometimes from different replication strategies \cite{curth2021really}. To generate a fair comparison between all algorithms, we followed the train/test split strategy with 1000 replications provided in the original study \cite{louizos2017cevae}, and calculate the square root of the Precision in Estimation of Heterogeneous Effect (rPEHE) for all included state-of-the-art algorithms.
For reproducibility, the results of TVAE as well as existing algorithms are provided using Jupyter Notebook on Github: \url{https://github.com/xuebing1234/tvae}. As tabulated in Table 3, TVAE has significant improvement in estimating treatment effects over the state-of-the-art models. 

Since both the factual and counterfactual outcomes are available, we further evaluate how different Distribution Matching strategies (Kernelized MMD, MMD, Wasserstein Distance, and KL divergence) affect the treatment effect estimation. Among them, Kernelized MMD and and KL divergence lead to lowest estimation errors, indicating that they might be more suitable for distribution matching in batch learning. 
It is noteworthy that IHDP (and potentially any other synthesized datasets) does not possess the complexity in real ECMO scenario. For instance, the level of scarcity in the treatment assignment (pos/neg ratio 23\%) is >10 times higher than ECMO datasets, and it has much smaller selection bias (derandomization simply by removing non-white mothers) and simpler input space.

\begin{table}[t]
  \caption{IHDP Dataset Performance (*: Results reported in the original paper that used the same replications.).}
  \label{Table3}
  \centering
  \fontsize{8pt}{9pt}\selectfont
  \begin{tabular}{lc}
    \toprule
    \textsc{Model} & $\epsilon_{PEHE}$\\
    \midrule
    OLS1 & 7.59 ± 0.33   \\
    OLS2 & 2.33 ± 0.11   \\
    BART & 1.98 ± 0.09  \\
    Causal Forest & 4.18 ± 0.20 \\
    KNN  & 3.62 ± 0.15\\
    GANITE & 1.83 ± 0.01\\
    CEVAE*   & 2.60 ± 0.10\\
    BNN*  &  2.10 ± 0.10\\
    DCNPD   & 2.05 ± 0.03\\
    TNet    &   1.77 ± 0.05\\
    SNet    &   1.29 ± 0.04\\
    TARNET    &   1.24 ± 0.04\\
    Dragonnet    &   1.39 ± 0.05\\
    \midrule
    \textbf{TVAE} + \text{kMMD}    &   \textbf{1.18 ± 0.04}\\
    \textbf{TVAE} + \text{MMD} &   \underline{1.19 ± 0.04}\\
    \textbf{TVAE} + \text{Wasserstein} &   1.20 ± 0.04\\
    \textbf{TVAE} + \text{KL} &   \textbf{1.18 ± 0.04}\\
    
    \bottomrule
  \end{tabular}
\end{table}
\section{Conclusion and Broader Impact}
We aim to support the challenging treatment decisions on ECMO treatment, a scarcely available life-supporting modality for COVID-19 patients. We proposed a disentangled representation model that delivers the propensity score and potential outcomes through semi-supervised variational autoencoding. Several innovative components are proposed and integrated to address the strong selection bias, scarcity of treatment cases, and the curse of dimensionality in patient characterization. The Treatment Joint Inference eliminates the auxiliary networks while regulating factual and counterfactual predictions through input reconstruction, clustering and semi-supervision. The Distribution Balancing disentangles the latent dimensions into different aspects of information and extract the selection bias to aid propensity scoring. The remaining non-confounding dimensions are regulated by maximizing the distributions between treatment and control. The Label Balancing component mitigates data imbalance with generative latent representation while preserving the patient distributions. The experiments on two real-world COVID-19 datasets and the public synthetic dataset show that the model is robust in capturing the underlying decision-making process as well as individual treatment effects, hence outperforming other state-of-the-art algorithms.

\textbf{Potential Impact:}
This work fills the gap between treatment effect models and the critical need for ECMO clinical decision support (or other application domains that face strong selection bias, scarcity of treatment, and overfitting). To our best knowledge, there is no machine learning tool that can aid clinicians in identifying ECMO candidates from high-dimensional EHR data while weighing the risks of disease progression versus the scarcity of ECMO support. 
In fact, the decision is arguably the most complex decision made in the ICU setting. 

The improvement in predictions over state-of-art models leads to the better identification of treatment needs and resource allocation, hence saving lives. When using a machine learning assistive tool in ICU setting, we usually want the model to maximize sensitivity (true positive rate) while fixing a high specificity (true negative rate). Comparing TVAE with existing algorithms (for example, BART), when fixing a high specificity of 0.95, TVAE has a sensitivity of 0.84 (while BART has a sensitivity of 0.71) in the BJC Dataset, and 0.64 (while BART has 0.33) in ISARIC Dataset. Considering the capacity of ECMO resources, such performance improvement leads to 374 more patients (18 in BJC Dataset, and 356 in ISARIC) being correctly identified for ECMO treatment (see more details in Appendix F). Although it is impossible to compare treatment effect estimation without counterfactual outcomes, the improved AUROC and AUPRC in factual outcome prediction suggests potential improvement in characterizing the survival/death impact for each treatment decision.  

\textbf{Limitation:}
This work is not without limitations. In clinical practice, when the treatment decisions are not EHR-related or tracked by the EHR system, the assumption of No Unmeasured Confounding might be violated. Meanwhile, the evaluation of model performance on the counterfactual outcomes remains to be a challenge except using synthetic datasets. Due to the concerns in potential algorithm bias and robustness arose from the data collection, sample size, and underestimation of certain groups, the machine learning predictions should only serve as assistance to clinical decision making.  In health care settings, the treatment effect estimated by our model should be used as only one of the inputs to the clinicians alongside other clinical and ethical considerations. How to incorporate our treatment effect model in clinical decisions should be investigated in future studies.
\begin{acks}
The completion of this research project would not have been possible without the contributions and support of the ISARIC Clinical Characterisation Group (ORCID ID: 0009-0004-5601-9672), Pandemic Sciences Institute, University of Oxford. We are deeply grateful to all contributors who played a role in the success of this project. The full author list and funding information is provided in the permanent \href{https://docs.google.com/document/d/17pWtnRI251dDsaQPAb5SMFjHQgEkJ-Hi/edit?usp=sharing&ouid=114308118763539217087&rtpof=true&sd=true}{\underline{\textit{link}}}. 

    This work was supported by the Fullgraf Foundation, and the Big Ideas 2020 COVID Grant through the Healthcare Innovations Lab at the BJC Healthcare and Washington University in St. Louis School of Medicine. ASS has received research support from the Children’s Discovery Institute Faculty Development Award at Washington University in St. Louis.
\end{acks}
\bibliographystyle{ACM-Reference-Format}
\balance
\bibliography{bib}
\appendix
\section{Discussion on assumptions}
The decision to initiate ECMO is almost always the most complex decision to make. ECMO is the most resource intensive therapy provided in the ICU and is associated with significant morbidities, in addition there are no universally accepted tools to identify patients at highest risk of receiving ECMO or universally accepted criteria for ECMO initiation. It was thus important to first identify patients not eligible for ECMO (by either BMI or age). From an inclusion perspective it was then important to include all the clinical, laboratory and therapeutic variables that influence ECMO decision making as no single variable can solely identify patients who might or might not receive ECMO. For example, a patient’s respiratory rate if high could represent severe respiratory distress and failure that could contribute to the decision to provide ECMO. A normal respiratory rate for a patient who is endotracheally intubated, on mechanical ventilatory support and under neuromuscular blockade but without satisfactory evidence of adequate gas exchange could also be considered a reason to consider ECMO support. Additionally, a patient in severe respiratory failure with subsequent hypercarbia can become hypopneic with low respiratory rate could be at risk of impending cardio-respiratory arrest and thus could be a candidate for urgent ECMO support.

\subsection{Discussion on Assumption 1}
The no unmeasured confounding assumption comes from the fact that an ICU patient is continuously measured in various aspects that might be relevant to treatment assignment and potential treatment outcomes, and clinicians rely on these measures to make reasonable treatment decisions. We assume that clinicians have taken necessary measurements that reflect all confounding information. As all such measurements/tests are captured in the Electronic Health Records, we believe the dataset has included all confounding variables. Take Institutional Data for example, we collect 52,216 measurements only from the flowsheets table, despite other tables such as demographics, comorbidities, lab tests, ventilation settings, etc. However, in the actual implementation, highly missing measures/tests are excluded from the model inputs, which remains to be investigated if this incurs any no unmeasured confounding. This has been included in the discussion of limitations in Sec. 6. 

\subsection{Discussion on Assumption 2}
For the rigorous analysis, we pre-exclude definite cases with zero probability (for example, any ICU patients with age>80 is considered unsuitable for ECMO treatment). Since there are no clinical criteria that an ICU patient ‘must’ be on ECMO, probability of each invidual treatment assignment is less than 1. In the resulting dataset, the probability of treatment decision is always in the range of (0,1). 

\section{Identifiability of Treatment Outcomes}
To predict the factual and counterfactual outcomes (hence the individual treatment effect), we need to show that $p(Y|X,\textit{do}(W=1))$ in TVAE is identifiable. From the Identification Theorem \cite{louizos2017cevae}, we can see that: 
\begin{align*}
    p(Y|X,\textit{do}(W=1))&=\int_{Z} p(Y|X,\textit{do}(W=1),Z)p(Z|X,\textit{do}(W=1))dZ\\
    &\stackrel{i}{=}\int_{Z} p(Y|X,W=1,Z)p(Z|X,W=1)dZ\\
    &\stackrel{j}{=}\int_{Z^1} p(Y|X,W=1,Z^1)p(Z^1|X,W=1)dZ^1\\
    &\stackrel{k}{=}\int_{Z^1} p(Y|W=1,Z^1)p(Z^1|X,W=1)dZ^1
\end{align*}
where equality $(i)$ is by the rules of \textit{do}-calculus, equality $(j)$ is by the designated latent dimension in TVAE, and equality $(k)$ comes from the property of VAE that $Y$ is independent of $X$ given $Z$. To ensure that $Y$ is only expressed in the designated latent dimension, disentanglement is further added to TVAE (see Sec 4.3). The case of $p(Y|X,\textit{do}(W=0))$ is identical, hence the defined individual treatment effect can be recovered.  
\section{Proof of Proposition 1}
\begin{proposition*}
    Given $Z^{1:d}:=\{Z^1, Z^2, ..., Z^d\}$ and $Z^1 = \hat{W}$, conditioned on the data $X$, the total correlation of $Z^{1:d}$ equals the sum of the total correlation of $Z^{2:d}$ and the mutual information between the first dimension $\hat{W}$ and all the other dimensions $Z^{2:d}$, i.e., 
\begin{equation}
        TC(Z^{1:d}|X) = TC(Z^{2:d}|X) 
    + I(\hat{W}|X;Z^{2:d}|X)
\end{equation}
where $I(A;B)$ is the mutual information between variable $A$ and $B$.
\end{proposition*}

\begin{proof}
\begin{equation}
\begin{split}
    LHS &=\mathbb{E}_{q_{\phi}(Z^{1:d}|X)} \bigl[ \log \frac{q_{\phi}(Z^{1:d}|X)}{q_{\phi}(Z^1|X)q_{\phi}(Z^2|X)...q_{\phi}(Z^d|X)}  \bigl] \\
    &= \mathbb{E}_{q_{\phi}(Z^{1:d}|X)} 
    \Bigl[ \log \bigl(\frac{q_{\phi}(Z^{2:d}|X)}
    {q_{\phi}(Z^1|X)q_{\phi}(Z^2|X)...q_{\phi}(Z^d|X)}\bigr) \\
    &+ \log \bigl(\frac{q_{\phi}(Z^{1:d}|X)}
    {q_{\phi}(Z^{2:d}|X)q_{\phi}(Z^1|X)}\bigr)\Bigr] \\
    &= TC(Z^{2:d}|X) 
    + I(Z^1|X;Z^{2:d}|X) \\
    &= RHS 
\end{split}
\end{equation}
\end{proof}

\section{Predicted risks of control and treatment cases}
The predicted distribution of ECMO cases and control cases in terms of treatment effect and control effect are plotted in Figure S1. Patients positioned to the right have higher mortality risk even without ECMO treatment, and patients in the upper regions have higher mortality risks after ECMO treatment. From the figure we can see that 1): Death ECMO cases have higher predicted treatment and control mortality risks than survival cases, and 2): actual treatment cases do not have the highest predicted mortality risks (either treatment or no-treatment).  
For 1), this is consistent with our intuition. First, the deaths after treatment demonstrate their high mortality risk even after ECMO. Second, the deaths imply their more severe symptoms before treatment decision, hence they are associated with higher mortality risks without ECMO treatment.  For 2), the actual clinical decision is not the reduction in mortality risks, but the changes in outcome. For instance, a reduction of death probability from 30\% to 0 will not justify the ECMO treatment when comparing to the reduction of death probability from 65\% to 35\%, since the latter is more likely to result in a change of outcome (changing the patient from death to survival). As a result, the actual ECMO assignment are not picking the cases with highest mortality risk (more likely to die regardless of treatment or not), but rather the cases that are more likely to be overturned. 


\begin{figure}[h]
    \centering
    \includegraphics[width=0.7\linewidth]{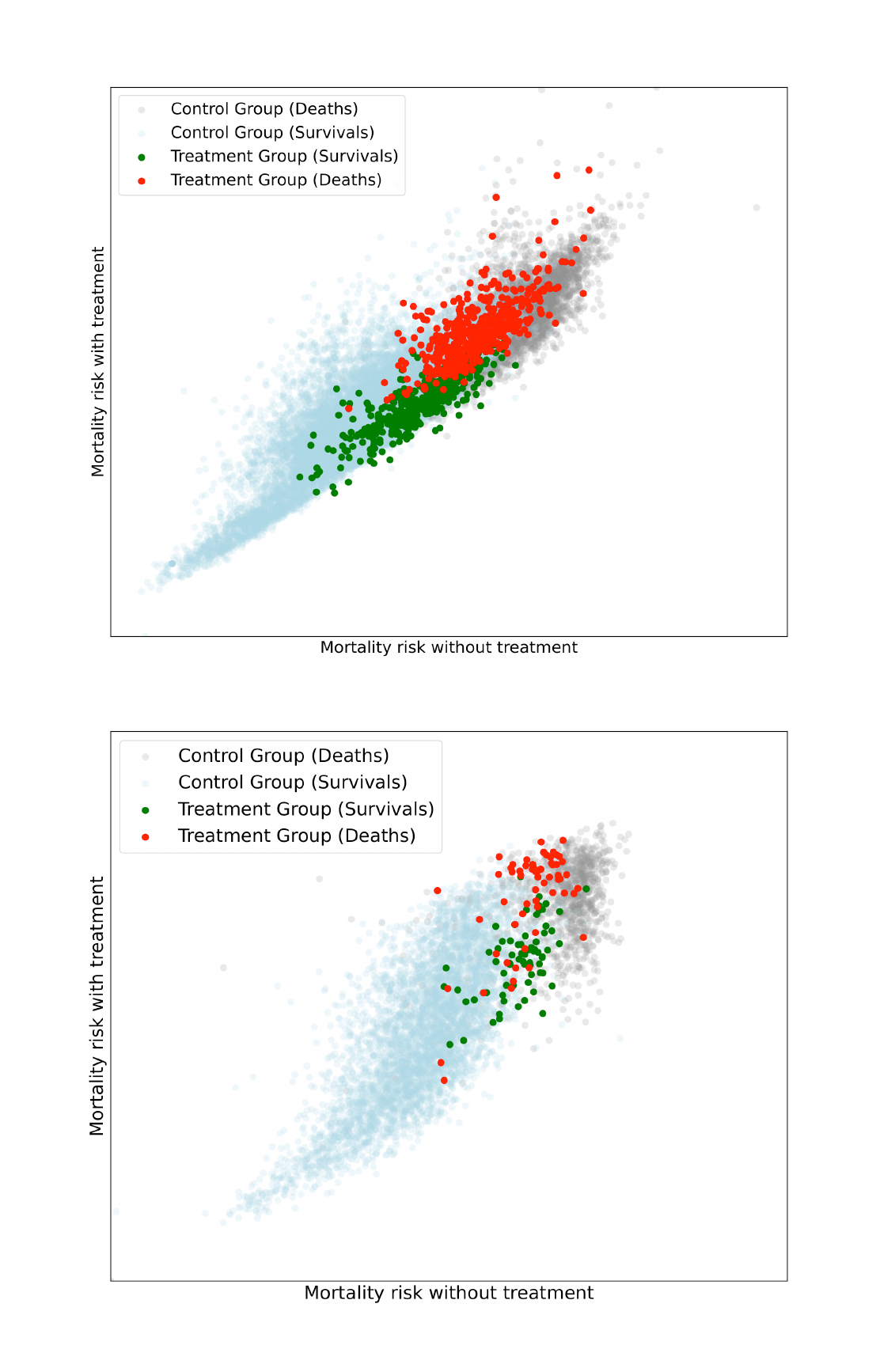}
\captionsetup{labelformat = empty } 
\caption{Figure S1: Distribution of treatment and no-treatment mortality risks in control (deaths), control (survivals), treatment (death) and treatment (survivals) groups. Left: ISARIC; Right: BJC Dataset}
\end{figure}
\section{data access agreement, preprocessing and cohort characteristics}
Data access agreement with International Severe Acute Respiratory and emerging Infections
Consortium (ISARIC) were signed on Dec 18, 2020. The purpose of access is to contribute to execute an analysis of \textbf{"Development of predictive
analytics model for need of extracorporeal support in COVID-19"}. 

For the BJC dataset, the IRB (\#202011004) titled \textbf{"Identifying predictors for ECMO need in COVID 19 patients"} was approved on Nov 2, 2020. 

Detailed data description and processing pipeline and characteristics of cohort are summarized. For ISARIC, the summary is provided here: https://tinyurl.com/mrv8rmp3. For the institutional data, the summary is provided here: https://tinyurl.com/yxersh7d.
\section{More metrics between TVAE and BART}

\begin{table}[tbh!]
  \caption{More point metrics between TVAE and BART}
  \label{TableS1}
  \centering
  \fontsize{7.5pt}{9pt}\selectfont
  \begin{tabular}{lccccc}
    \toprule
    \textsc{Model} & \textsc{Sensitivity}  &\textsc{Specificity}   & \textsc{Precision}  &\textsc{Accuracy} 
    &\textsc{F1} 
    \\
    \midrule
    \multicolumn{6}{c}{Institutional Dataset}\\
    \midrule 
    BART & .7100 & .9531  & .2563& .9476 &.3752   \\
    TVAE & .8356 & .9517& .2887 & .9491 &.4275   \\
    \midrule
    \multicolumn{6}{c}{ISARIC Dataset}\\
    \midrule
    BART & .3330 & .9499 & .0609 & .9439 & .1029 \\
    TVAE  & .6353 & .9506 &.1114  &.9476  &.1895 \\
    \bottomrule
  \end{tabular}
\end{table}
\section{Parameters of baselines}
The summary of hyper-parameters used in ECMO datasets are listed here: https://tinyurl.com/433yhjzh The hyper-parameters are tuned through grid-search. For IHDP, we use the default hyper-parameters in the associated Github code repository, and uploaded the experiment results here: https://github.com/xuebing1234/tvae
\section{Optimal Sampling strategy}
\begin{figure}[h]
\captionsetup{labelformat=empty }
 \captionsetup{width=.4\columnwidth}
\allowbreak
\begin{minipage}[t]{.65\columnwidth}
    \centering
    \includegraphics[width=1\columnwidth]{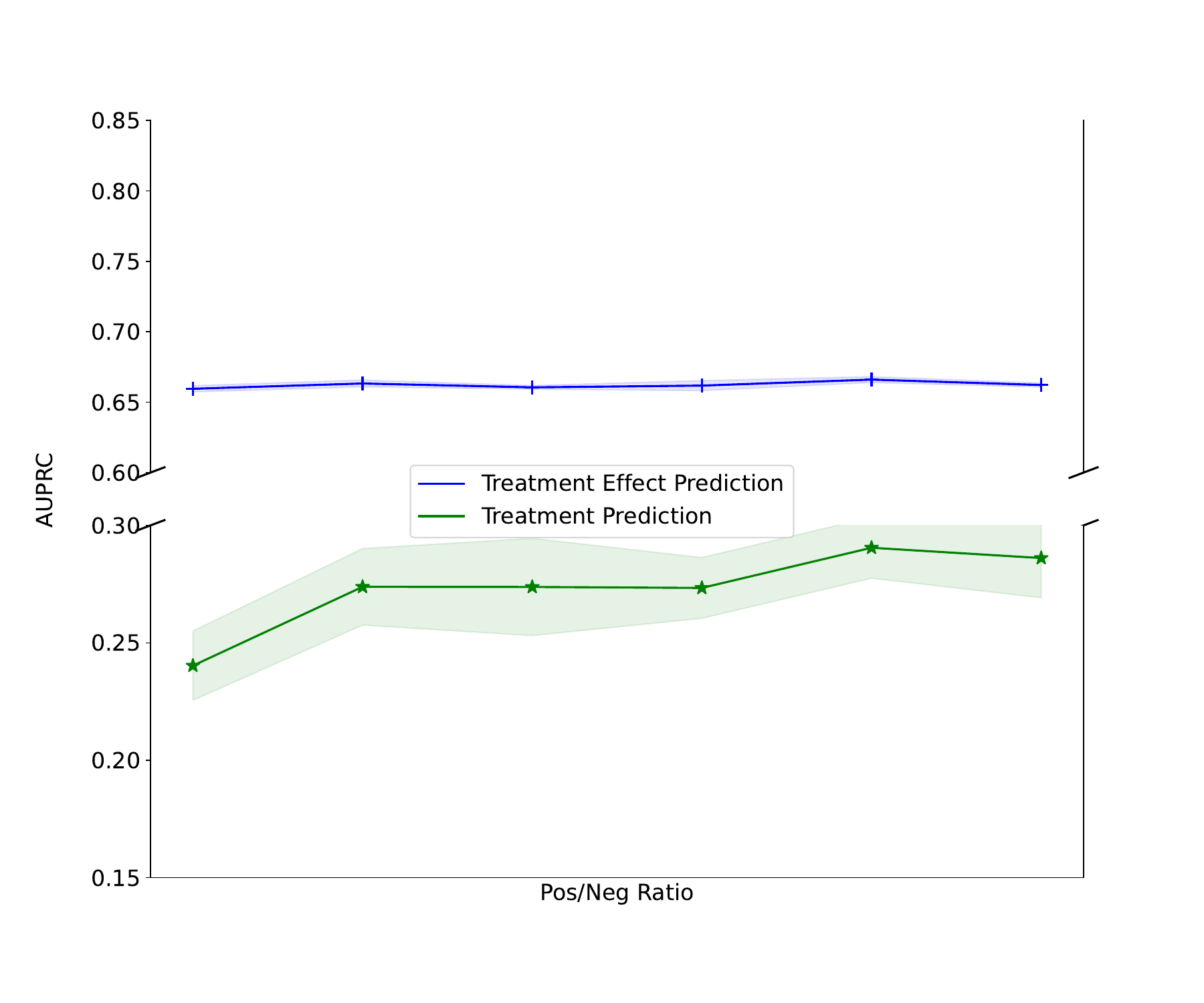}
\end{minipage}
\par\vspace{-1\baselineskip}
\begin{minipage}[t]{.65\columnwidth}
    \centering
    \includegraphics[width=1\columnwidth]{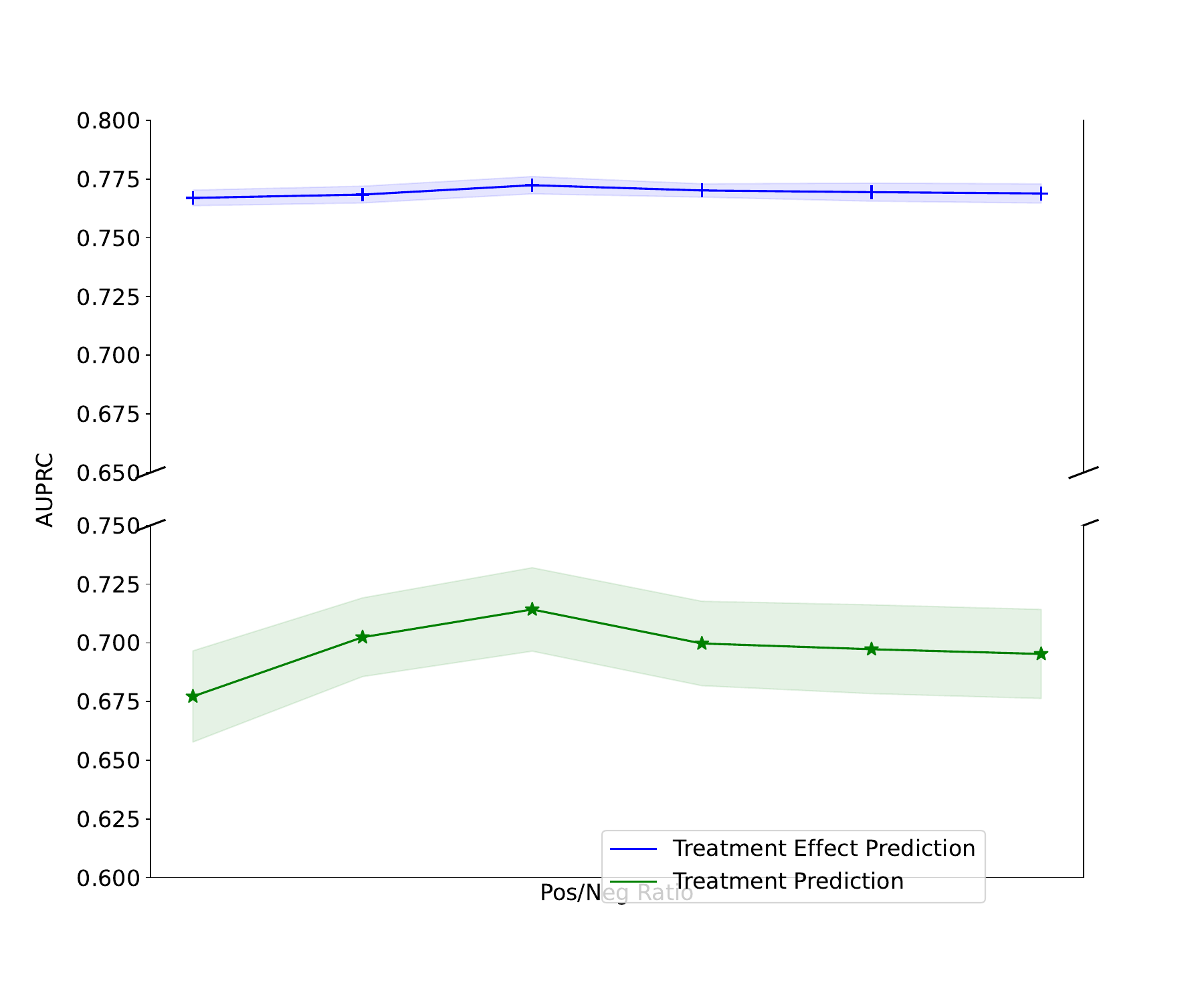}
\end{minipage}
\captionsetup{width=1\linewidth}
\captionsetup{labelformat = empty } 
\caption{Figure S2: Prediction performance (measured by AUPRC). Upper: ISARIC Data; Lower: BJC Data.}
\end{figure}
\clearpage

\end{document}